\documentclass[letterpaper, 10pt, conference]{ieeeconf} 

\usepackage{balance}
\usepackage{mathrsfs}
\usepackage[latin1]{inputenc}
\usepackage{mathtools}
\usepackage{amssymb}
\usepackage{amsmath}
\usepackage{amsfonts}
\usepackage{balance}
\usepackage{authblk}    
\usepackage{makeidx}    
\usepackage{graphicx}       
\usepackage{amsfonts}
\usepackage{psfrag}
\usepackage{setspace}
\usepackage{epstopdf}
\usepackage{bm}
\usepackage{changes}
\usepackage{nicefrac}
\usepackage{adjustbox}
\usepackage{pgf-pie}
\usepackage{fixmath}
\usepackage{soul}
\usepackage{epsfig}
\usepackage{array}
\usepackage{graphics}
\usepackage{relsize}
\usepackage{arydshln}
\usepackage{algorithm}
\usepackage{algpseudocode}
\usepackage{mathtools}

\renewcommand{\baselinestretch}{0.95}

\DeclareMathOperator{\tr}{tr}
\newtheorem{dfn}{Definition}

\newtheorem{lemma}[dfn]{Lemma}

\usepackage{bbm}

\IEEEoverridecommandlockouts
\begin{document}

\title{Adaptive Out-of-Control Point Pattern Detection in \\Sequential Random Finite Set Observations}

\author{Konstantinos~Bourazas$^\dagger$, Savvas~Papaioannou, and Panayiotis~Kolios
\thanks{The authors are with the $\prescript{\dagger}{}{\text{Department of Economics}}$, Athens University of Economics and Business, Athens, Greece, and the KIOS Research and Innovation Centre of Excellence (KIOS CoE), University of Cyprus, Nicosia, 1678, Cyprus. E-mail :{ \tt\small kbourazas@aueb.gr, \{papaioannou.savvas, pkolios \}@ucy.ac.cy}. %
This work is implemented under the Border Management and Visa Policy Instrument (BMVI) and is co-financed by the European Union and the Republic of Cyprus (GA:BMVI/2021-2022/SA/1.2.1/015), and also supported by the European Union's H2020 research and innovation programme under grant agreement No 739551 (KIOS CoE - TEAMING) and through the Deputy Ministry of Research, Innovation and Digital Policy of the Republic of Cyprus.
}}

\maketitle

\begin{abstract}
In this work we introduce a novel adaptive anomaly detection framework specifically designed for monitoring sequential random finite set (RFS) observations. Our approach effectively distinguishes between In-Control data (normal) and Out-Of-Control data (anomalies) by detecting deviations from the expected statistical behavior of the process. The primary contributions of this study include the development of an innovative RFS-based framework that not only learns the normal behavior of the data-generating process online but also dynamically adapts to behavioral shifts to accurately identify abnormal point patterns. To achieve this, we introduce a new class of RFS-based posterior distributions, named Power Discounting Posteriors (PD), which facilitate adaptation to systematic changes in data while enabling anomaly detection of point pattern data through a novel predictive posterior density function. The effectiveness of the proposed approach is demonstrated by extensive qualitative and quantitative simulation experiments.
\end{abstract}

\section{Introduction}

An anomaly is an observation or event that significantly
diverges from expected patterns. Anomaly detection encompasses
a range of methodologies aimed at
monitoring processes (e.g., manufacturing, industrial,
and financial) and identifying such deviations, which
may indicate potential problems or inefficiencies. Anomaly
detection finds widespread application in diverse domains,
including fraud detection (credit cards, insurance, healthcare),
cybersecurity, and fault discovery in critical systems
among others. This broad applicability arises from the crucial,
actionable insights anomalies often reveal. For instance, unusual network traffic may signal security breaches \cite{Kumar2005}, abnormal medical findings may indicate malignancy \cite{Spence2001}, and atypical transactions may reveal financial fraud \cite{Bolton2002}. Despite decades of research and maturity in some domains, detecting anomalies remains challenging, particularly in stochastic point pattern data \cite{diggle2013statistical}, where both the number and nature of observations are random. Most existing statistical \cite{chandola2009anomaly,samariya2023comprehensive} and data-driven \cite{pang2021deep} methods assume fixed-size input structures (e.g., vectors, matrices), limiting their effectiveness in such settings.

Point patterns, however, consist of random sets of random vectors, with both spatial configuration and cardinality varying across instances. This variability renders traditional methods largely ineffective. To address this, Random Finite Set (RFS) theory \cite{mahler2014advances,mahler2007statistical} offers a robust framework for handling random set-valued data and parameters. RFS is particularly suited for inference tasks involving an unknown and varying number of observations, with applications in state estimation \cite{VoBaTuong,Wei2021,papaioannou2019jointly}, tracking \cite{BaNguVo2005,Wenling2011,papaioannou2019probabilistic}, robotics \cite{Doerr2018,Papaioannou2019,papaioannou2020cooperative}, and surveillance \cite{Yang2019,Papaioannou2021,papaioannou2020IROS}.

However, its application in anomaly detection has been largely overlooked to date. Related to our work are the works in \cite{vo2018model, kamoona2019random, kamoona2021point} and \cite{kamoona2024anomaly}. Specifically, the authors in \cite{vo2018model} address model-based learning for classification, anomaly detection, and clustering of point pattern data. The approaches proposed in \cite{vo2018model} are offline, i.e., they are not designed for observations obtained sequentially, and require a two-stage procedure. A training/calibration phase first learns the model's parameters, and only then is anomaly detection performed. Similarly, offline two-stage RFS-based anomaly detection algorithms are proposed in \cite{kamoona2019random} for hazard detection at construction sites, and in \cite{kamoona2021point} for detecting anomalies in industrial manufacturing. The problem of anomaly detection with RFS observations is also investigated in \cite{kamoona2024anomaly}, where the authors propose a methodology for detecting defects in industrial settings using RFS energy-based models. However, this methodology is not adaptive; the RFS energy-based model is known a priori, and its parameters, learned using offline calibration data, remain fixed thereafter. In contrast, our proposed approach operates within the Bayesian framework, where the model parameters are themselves treated as random variables following (non-informative) conjugate prior distributions, and therefore can be adapted to match the normal behavior of the generating process.

Specifically, we propose a novel adaptive anomaly detection framework for sequential random finite set (RFS) observations. We assume an unknown Poisson RFS process generates sequential point pattern data. Our system learns and monitors the process in an online fashion while determining whether the data remains in statistical control 
\cite{qiu2013introduction}. Data conforming to the expected behavior are considered In-Control (IC), or normal, while deviations signify Out-Of-Control (OOC) data, or anomalies. Anomalies indicate either a temporary deviation from the in-control process behavior or a fundamental shift, necessitating an update to the established process statistical state. Our contributions are the following:

\begin{itemize}
    \item We investigate the problem of anomaly detection with sequential random finite set (RFS) observations, and propose an adaptive anomaly detection framework. This framework learns the process's normal behavior in an on-line fashion, adapting to behavioral shifts, while being capable of detecting out-of-control (OOC) point pattern data (i.e., anomalies) that do not conform to the expected behavior of the process.


    
    \item We introduce the Power Discounting Posteriors (PD), a new class of RFS-based posterior distributions designed for on-line learning and process monitoring. Subsequently, we derive a Bayesian RFS predictive check statistic, allowing robust detection of OOC data points, or anomalies.

    \item Finally, the effectiveness of the proposed approach is demonstrated through extensive qualitative and quantitative simulation experiments.
\end{itemize}

The paper is organized as follows. In Section \ref{sec:preliminaries}, we provide the background on the theory of random finite sets. In Section \ref{sec:problem}, we formulate the problem tackled in this work, and subsequently, in Section \ref{sec:approach}, we discuss the details of the proposed approach. Finally, in Section \ref{sec:sim}, we evaluate the proposed approach, and in Section \ref{sec:disc}, we conclude the paper.


\section{Preliminaries} \label{sec:preliminaries}

\subsection{Random Finite Sets} \label{subsec:rfs}
A random finite set (RFS) is a finite-set-valued random variable where both the number of elements in the set and the values of these elements are random. Therefore, the main difference between a random finite set and a random vector lies in two aspects for the former: firstly, the number of elements, denoted as $n$, is itself random, and secondly, these elements are not only random but also unordered.

More specifically, an RFS $X$ is completely specified by: a) its cardinality distribution $\rho(n) = p(|X|=n),~ n \in \mathbb{N}$, which defines the probability distribution over the number of elements in $X$, and b) by a family of conditional joint probability distributions $p(x_1,\ldots,x_n|n)$ that characterize the distribution (i.e., spatial density) of its elements $x_1,\ldots,x_n \in \mathcal{X}$ over the state space $\mathcal{X}$. The probability density function (pdf) $f(X)=f(\{x_1,\ldots,x_n\})$ of the RFS $X$, with $f: \mathcal{F}(\mathcal{X}) \rightarrow [0,\infty)$ where $\mathcal{F}(\mathcal{X})$ is the space of all finite subsets of $\mathcal{X}$, is given by:
\begin{equation} \label{eq:rfs_lik1}
    f(\{x_1,\ldots,x_n\}) = \rho(n) \sum_{\varpi} p(x_{\varpi(1)},\ldots,x_{\varpi(n)}|n)
\end{equation}

\noindent where $\varpi$ is the permutation of the elements $\{1,\ldots,n\}$, and is used as shown above to account for the fact that the elements in $X$ are unordered. Moreover, for joint symmetric conditional densities, Eq.~\eqref{eq:rfs_lik1} simplifies to:

\begin{equation} \label{eq:rfs_lik2}
    f(\{x_1,\ldots,x_n\}) = \rho(n) n! p(x_{1},\ldots,x_{n}|n)
\end{equation}

\noindent where $n!$ denotes the $n$-factorial, and $p(x_{1},\ldots,x_{n}|n)$ is a symmetric density, meaning that its value remains unchanged for all of the possible $n!$ permutations of its input variables. The notion of integration is then given by the set-integral, which is defined as:
\begin{equation}
    \int f(X) d X = f(\emptyset) + \sum^{\infty}_{n=1} \frac{1}{n!} \int f(\{x_1,\ldots,x_n\}) \, dx_1 \ldots dx_n
\end{equation}
\noindent where by convention $f(\emptyset) = \rho(0)$. Finally, it is straightforward to show that when the elements $x \in X$ of the RFS $X$ are independent and identically distributed (iid) according to the probability density $p(x)$ on $\mathcal{X}$, the pdf of $X$ is given by:
\begin{equation} \label{eq:lik1}
    f(X) = \rho(n) n! \prod_{x \in X} p(x)
\end{equation}
Additionally, when the cardinality $n$ follows a Poisson distribution with parameter $\lambda$, referred to as Pois($\lambda$), i.e., the cardinality distribution is given by $\rho(n) = \frac{e^{-\lambda}\lambda^n}{n!}$, the RFS $X$ becomes a Poisson RFS with density given by:
\begin{equation}\label{eq:RFS_lik_Poisson}
    f(X) = e^{-\lambda}\prod_{x \in X} \kappa(x)
\end{equation}
\noindent where $\kappa(x) = \lambda p(x)$ is called the intensity function of $X$, which when integrated over any closed subset $S \subseteq \mathcal{X}$ gives the expected number of elements $\mathbb{E}[n]$ in $S$ i.e., $\mathbb{E}[n] = \int_S \kappa(x) dx$. The Poisson RFS, due to its importance in diverse application scenarios \cite{streit2010poisson}, serves as the primary focus of this work for anomaly detection. However, the proposed approach can be generalized in other RFS models. It is also worth mentioning that the non-uniformity of the reference measure in the RFS framework causes problems in anomaly detection \cite{vo2018model}. To overcome this problem, the ranking function has been proposed \cite{vo2018model}, which for the iid cluster model is given by:
\begin{equation} \label{eq:rank0}
r(X) \propto \rho(n) \dfrac{\displaystyle{\prod_{x \in X}} p(x)}{\left(\lVert p(x) \rVert_2^2\right)^n},
\end{equation}
where $\lVert \cdot \rVert_2$ is the $L_2$ norm. 

\section{Problem Formulation} \label{sec:problem}

We consider a discrete-time, bounded Poisson RFS process with density governed by Eq. \eqref{eq:RFS_lik_Poisson}. At each time-step $t$, the process generates an RFS observation $X_t = {x_{1,t},\ldots,x_{n,t}}$, where $n_t \sim \text{Pois}(\lambda_t)$ and $x_{i,t} \sim N_{d}(\mu_t, \Sigma_t)$. Anomalies arise from sudden and large changes in the rate parameter or mean vector.  The problem tackled in this work can be stated as follows: 

\textit{Given sequential Poisson RFS observations $\{X_t| t>0\}$, our objective is to dynamically learn the posterior predictive density $f(X_{t+1}|X_{1:t})$ for the subsequent time-step $t+1$, based on all preceding observations $X_{1:t}$ up to time $t$. Following this, we aim to assess the ``extremeness'' in terms of the statistical significance for the newly received RFS observation $X_{t+1}$ in relation to the process' expected behavior, and detect OOC point pattern data (i.e., anomalies) that do not conform to the process' anticipated behavior.}


A high-level overview of the proposed framework is given next and detailed discussion in Sec. \ref{sec:approach}.
Within the Bayesian framework, we consider the rate parameter $\lambda_t$ along with the mean vector $\mu_t$ and the covariance matrix $\Sigma_t$ as unknown. We denote the list of the unknown parameters as $\Theta_t$ i.e., $\Theta_t=( \lambda_t, \mu_t, \Sigma_t)$, that belong in a parametric space $\mathbold{\Theta} = \mathbb{R}^{+} \times \mathbb{R}^{d} \times \mathbb{R}^{d \times d}$. By making use Eqs. \eqref{eq:lik1}-\eqref{eq:RFS_lik_Poisson}, in this work, the likelihood function $f(X_t | \Theta_t)$ of the Poisson RFS process can be computed as:
\begin{align}\label{eq:lik2}
    f(X_t | \Theta_t) &= e^{-\lambda_t} \prod_{j=1}^{n_t} \lambda_t p(x_{j,t} | \mu_t, \Sigma_t)
\end{align}

Our goal is to learn on-line the parameters $\Theta_t$ in Eq. \eqref{eq:lik2} through sequential RFS observations, aiming to ascertain the anticipated behavior of the process. Simultaneously, we seek to identify observations that deviate from this expected behavior (i.e., anomalies), by taking into account the uncertainty of $\Theta_t$. We formulate this uncertainty assuming that $\Theta_t$ follows a prior distribution, denoted as $\pi(\Theta_t)$. The prior distribution represents our belief on the parameters, with the non-informative priors \cite{gelman2013bayesian} being a plausible choice in cases of process ignorance. Given the RFS observations $X_{1:t}$ up to time-step $t$, we compute the posterior density, denoted as $\pi(\Theta_t|X_{1:t})$, which is given by:
\begin{align} \label{eq:post0}
	\pi\left( \Theta_t | X_{1:t} \right) & \propto f \left( X_{1:t} | \Theta_t \right)  \pi(\Theta_t), 
\end{align}
resulting in an updated version of the prior informed by the likelihood. Extending the framework of the classical posterior in the Bayesian approach, we introduce the Power Discounting (PD) posteriors $\pi(\Theta_t|X_{1:t}, \alpha_0)$, further discussed in Sec. \ref{ssec:bayes}. These allow for adaptation to changes in the process's behavior by learning the parameters $\Theta_t$ over time using a discounting factor $\alpha_0$. Subsequently, the posterior predictive distribution $f(X_{t+1} | X_{1:t}, \alpha_0)$, is given by:
{\small
\begin{align}
f(X_{t+1} | X_{1:t}, \alpha_0)  =  \int f(X_{t+1} | \Theta_t) \pi ( \Theta_t | X_{1:t}, \alpha_0) d \Theta_t,
\end{align}}

\noindent which derived by integrating out the unknown parameters with respect to the posterior. Based on this, to determine whether the received RFS observation $X_{t+1}$ at time-step $t+1$ conforms to the process's expected behavior (i.e., is normal or anomalous), we leverage the independence between the cardinality distribution $\rho(n_{t+1})$ and the conditional spatial density  $p(x_{1,t+1},\ldots,x_{n_{t+1},t+1}|n_{t+1})$, to express the posterior in closed form as $\pi(\Theta_t|X_{1:t}, \alpha_0) = \pi(\lambda_t|n_{1:t}, \alpha_0) \pi(\mu_t, \Sigma_t|X_{1:t}, \alpha_0)$. 
Thus, we can derive the posterior predictive distributions of the cardinality, and the spatial density as shown in Eq. \eqref{eq:predictivec}, and  Eq. \eqref{eq:predictivex} respectively:
\begin{align} \label{eq:predictivec}
\rho \left( n_{t+1} | n_{1:t}, \alpha_0 \right)  &=  \int \rho(n_{t+1} | \lambda_t) \pi(\lambda_t|n_{1:t}, \alpha_0) d \lambda_t,
\end{align}
\begin{align} \label{eq:predictivex}
p \left( \bar{x}_{t+1} | X_{1:t}, \alpha_0 \right)  &=  \int \int p(\bar{x}_{t+1} | \mu_t, \Sigma_t) \times \nonumber \\
&\times  \pi(\mu_t, \Sigma_t|X_{1:t}, \alpha_0) d \mu_t d\Sigma_t,
\end{align}

\noindent where $\bar{x}_{t+1} = \sum_{j=1}^{n_{t+1}}x_{j,t+1} / n_{t+1}$ is the sufficient statistic of the mean vector summarizing all the available information from the features. This derivation is   presented in Sec. \ref{ssec:nov}.

Subsequently, we employ posterior predictive checks to assess the statistical significance of a new observation relative to the observed sequence of observations. Specifically, we define two posterior probabilities, denoted as $pr^n_{t+1}$ for the cardinality, and $pr^{x|n}_{t+1}$ for the features given cardinality at time $t+1$, which will play the role of posterior predictive $p$-values \cite{meng1994posterior}, measuring how extreme the observation is for the posterior prediction.

Finally, we assess the presence of an anomaly via the  Fisher's combined probability test \cite{fisher1970statistical} which in this work can be defined as:
\begin{equation} \label{eq:fisher}
    P_{t+1} = - 2  \log \left( pr^{n}_{t+1} \cdot pr^{x|n}_{t+1} \right).  
\end{equation}
For uniformly distributed probabilities $pr^{n}_{t+1}$ and $pr^{x|n}_{t+1}$ (discussed in detail in Sec. \ref{ssec:nov}), the distribution of $P_{t+1}$ under the null hypothesis of non-anomaly follows a chi-squared distribution with $4$ degrees of freedom, i.e., $P_{t+1} \sim \mathcal{X}^2_{4}$. Thus, we raise an alarm if:
\begin{equation} \label{eq:alarm}
    P_{t+1} >  q(1-\alpha), 
\end{equation}
where $q(\cdot)$ is the quantile function of  $\mathcal{X}^2_{4}$ distribution, and $\alpha$ is a predetermined false alarm rate $\alpha$.





\section{Adaptive Anomaly Detection in Sequential RFS Observations} \label{sec:approach}
\subsection{Power Discounting posteriors for RFS Point-pattern Data} \label{ssec:bayes}

In a sequentially updated process, as the observed horizon expands, the posterior distribution becomes more informative. While this can be advantageous in some cases, it may face difficulties in accurately adaptating small systematic changes. Specifically, the accumulation of extensive information on specific values leads to an ``inertia'' effect, wherein the posterior parameters require numerous observations to undergo changes and update to new values. The adaptation process becomes difficult as the posterior distribution becomes increasingly resistant to change. 

For this reason, we introduce a new class of posterior distributions, the Power Discounting posteriors (PD). PD posteriors, via a discounting factor $\alpha_0$, weigh the importance of the received observations, giving more weight to the most recent. The idea of power discounting is not new in Bayesian statistics but in a different set-up \cite{ibrahim2000power}, \cite{west1986bayesian}. PD posteriors differentiate, as they implement sequential power discounting; they are suitable for monitoring, but at the same time, they leverage the temporal significance of the observations to facilitate adaptation to systematic changes, and improve the anomaly detection capability. For the unknown parameters $\Theta_t$, the general form of the PD posterior at time $t$ is: 
\begin{align} \label{eq:post}
	\pi\left( \Theta_t | X_{1:t}, \alpha_0 \right) & \propto f \left( X_{t} | \Theta_t \right) \prod_{i=1}^{t-1} f \left( X_{i} | \Theta_t \right)^{\alpha_0^{t-i}} \hspace{-2.5mm} \pi(\Theta_t)^{\alpha_0^t}, 
\end{align}
\noindent 
The parameter $0 \leq \alpha_0 \leq 1$ is a subjective choice that defines the influence of recent observations on the posterior estimate, with their weight decaying exponentially. It allows the process to account for small, systematic changes, while still detecting anomalies. We preserve conjugacy for computational efficiency, but estimating $\alpha_0$ online from the data is a natural extension. In extreme cases, $\alpha_0=0$ results in pure adaptation, where only the latest observation influences the posterior, while $\alpha_0=1$ assigns equal weight to all observations, making the posterior more informative for history-based monitoring.

Because of the independence between $\lambda_t$ and $\left( \mu_t, \Sigma_t\right)$, their joint prior can be  written as  $\pi \left( \lambda_t, \mu_t, \Sigma_t\right)= \pi \left( \lambda_t \right) \pi \left( \mu_t, \Sigma_t\right)$ (this product holds for the posterior as well, i.e., $\pi(\lambda_t, \mu_t, \Sigma_t|X_{1:t}, \alpha_0) = \pi(\lambda_t|n_{1:t}, \alpha_0) \pi(\mu_t, \Sigma_t|X_{1:t}, \alpha_0)$). 

Lemma \ref{lem:pois} and Lemma \ref{lem:norm} derive the posterior distribution of $\lambda_t$ and the joint posterior distribution of $\left( \mu_t, \Sigma_t\right)$, respectively. 

 \begin{lemma} \label{lem:pois}
    Assuming the prior $\lambda_t \sim G\left(c_0, d_0\right)$, i.e., a Gamma distribution with hyperparameters $c_0$, and $d_0$, then the PD posterior distribution is conjugate, i.e., $\lambda_t | \left(  n_{1:t}, \alpha_0 \right) \sim G\left(c_t, d_t\right)$, with $c_t$, and $d_t$, the updated parameters. 
\end{lemma}

\begin{proof}
At time $t$, the posterior of the rate is: 
    \begin{align}
	\pi\left( \lambda_t | n_{1:t}, \alpha_0 \right) & \propto \rho \left( n_t | \lambda_t \right) \prod_{i=1}^{t-1} \rho \left( n_i | \lambda_t \right)^{\alpha_0^{t-i}} \pi(\lambda_t)^{\alpha_0^t}
\end{align}

Keeping only the expressions that contain the unknown parameters we have that $\pi\left( \lambda_t | n_{1:t}, \alpha_0 \right) \propto$:
\begin{align} 
    &  e^{-\lambda_t} \lambda_t^{n_t} \prod_{i=1}^{t-1} \left( e^{-\lambda_t}  \lambda_t^{n_i} \right) ^{\alpha_0^{t-i}} \left( e^{-d_0 \lambda_t} \lambda_t^{c_0} \right)^{\alpha_0^t} \nonumber \implies \\
		&\pi\left( \lambda_t | n_{1:t}, \alpha_0 \right)  \propto   e^{-\lambda_t d_t} \lambda_t^{c_t }, \nonumber
\end{align}
where $c_t=\alpha_0^t c_0  + n_t + \displaystyle{\sum_{i=1}^{t-1}} \alpha_0^{t-i} n_i $ and $d_t = \alpha_0^t d_0  + 1 + \displaystyle{\sum_{i=1}^{t-1}} \alpha_0^{t-i}$. Thus, $\lambda_t | (n_t, \alpha_0) \sim G(c_t, d_t)$.
\end{proof}

\begin{lemma} \label{lem:norm}
        Assuming the prior $\left( \mu_t, \Sigma_t\right) \sim NIW (m_0, l_0, \nu_0, \Psi_0)$, i.e., a Normal-Inverse Wishart with hyperparameters $m_0$, $l_0$, $\nu_0$, and $\Psi_0$, then the PD posterior distribution is conjugate $\left( \mu_t, \Sigma_t \right) | \left( X_{1:t}, \alpha_0 \right) \sim NIW (m_t, l_t, \nu_t, \Psi_t)$, with $m_t$, $l_t$, $\nu_t$, and $\Psi_t$, the updated parameters.
\end{lemma}
    
\begin{proof}
    At time $t$, the joint posterior of the mean vector and the covariance matrix is: 
{\footnotesize
\begin{align}
	\pi\left( \mu_t, \Sigma_t | X_{1:t}, \alpha_0 \right)  & \propto \prod_{j=1}^{n_t} p \left( x_{j,t} | \mu_t, \Sigma_t \right) \prod_{i=1}^{t-1}  \left( \prod_{j=1}^{n_i} p \left( x_{j,i} | \mu_t, \Sigma_t \right) \right)^{\alpha_0^{t-i}}  \\
 & \times   \pi(\mu_t, \Sigma_t)^{\alpha_0^t} \nonumber
 \end{align} }
Keeping only any expression that contains the unknown parameters we have that $\pi\left( \mu_t, \Sigma_t | X_{1:t}, \alpha_0 \right) \propto$
{\footnotesize
\begin{align}
   & | \Sigma_t |^{-\displaystyle{\sum_{i=1}^t}  \alpha_0^{t-i} n_i / 2} \hspace{-3mm}\exp \left\{-\dfrac{1}{2} \tr \Sigma_t^{-1} \displaystyle{\sum_{j=1}^{n_i}} x_{j,i} x_{j,i}^T \right\} \nonumber \\
	& \times \exp \left\{-\dfrac{1}{2} \left( \displaystyle{\sum_{i=1}^t}  \alpha_0^{t-i} n_i -2 \displaystyle{\sum_{i=1}^t} \alpha_0^{t-i} \displaystyle{\sum_{j=1}^{n_i}} x_{j,i}^T \Sigma_t^T \mu_t \right) \right\} \nonumber \\
 	& \times | \Sigma_t |^{-\displaystyle{\sum_{i=1}^{t-1}}  \alpha_0^{t-i} n_i / 2} \exp \left\{-\dfrac{1}{2} \tr \Sigma_t^{-1} \displaystyle{\sum_{j=1}^{n_i}} x_{j,i} x_{j,i}^T \right\} \nonumber \\
  	& \times \exp \left\{-\dfrac{1}{2} \left( \displaystyle{\sum_{i=1}^{t-1}}  \alpha_0^{t-i} n_i -2 \displaystyle{\sum_{i=1}^{t-1}} \alpha_0^{t-i} \displaystyle{\sum_{j=1}^{n_i}} x_{j,i}^T \Sigma_t^T \mu_t \right) \right\} \nonumber 
 \end{align}}
 {\footnotesize
\begin{align}
	& \times | \Sigma_t |^{ - {\alpha_0^t} (\nu_0 +d +2) / 2} \exp \left\{-\dfrac{1}{2} {\alpha_0^t} \tr \Sigma_t^{-1} \Psi_0 \right\} \nonumber \\	
	& \times \exp \left\{-\dfrac{ {\alpha_0^t} l_0 }{2} \left(  \mu_t  - m_0 \right)^T \Sigma_t^{-1} \left(  \mu_t  - m_0 \right) \right\} \nonumber 
 \end{align}}
 
By doing the appropriate operations and completing the quadratic forms we have:

{\footnotesize
\begin{align}
\pi\left( \mu_t, \Sigma_t | X_{1:t}, \alpha_0 \right)	& \propto | \Sigma_t |^{ -(\nu_t +d +2) / 2} \exp \left\{-\dfrac{1}{2} \tr \Sigma_t^{-1} \Psi_t \right\} \nonumber \\	
& \times \exp \left\{-\dfrac{\lambda_t}{2} \left(  \mu_t  - m_t \right)^T \Sigma_t^{-1} \left(  \mu_t  - m_t \right) \right\} \nonumber
\end{align}
}

where $m_t = \dfrac{\alpha_0^{t} l_0 m_0 + \displaystyle{\sum_{j=1}^{n_t}} x_{j,t} + \displaystyle{\sum_{i=1}^{t-1}} \alpha_0^{t-i} \displaystyle{\sum_{i=1}^{n_i}} x_{j,i} }{\alpha_0^{t}l_0 + n_t + \displaystyle{\sum_{i=1}^{t-1}} n_i \alpha_0^{t-i}}$,\\ $l_t = \alpha_0^{t}l_0 + n_t + \displaystyle{\sum_{i=1}^{t-1}} n_i \alpha_0^{t-i}$, $\nu_t = \alpha_0^{t}\nu_0 + n_t + \displaystyle{\sum_{i=1}^{t-1}} n_i \alpha_0^{t-i}$ and \\ $\Psi_t = \alpha_0^{t}\Psi_0 + \alpha_0^{t}l_0 m_0 m_0^T  + \displaystyle{\sum_{j=1}^{n_t}} x_{j,t}x_{j,t}^T + \displaystyle{\sum_{i=1}^{t-1}} \alpha_0^{t-i} \displaystyle{\sum_{j=1}^{n_i}} x_{j,i}x_{j,i}^T-$\\
$-\dfrac{1}{ \alpha_0^{t}l_0+ n_t + \displaystyle{\sum_{i=1}^{t-1}} n_i \alpha_0^{t-i}} \times \\  \left( \alpha_0^{t}l_0 m_0 + \displaystyle{\sum_{j=1}^{n_t}} x_{j,t} + \displaystyle{\sum_{i=1}^{t-1}} \alpha_0^{t-i} \displaystyle{\sum_{j=1}^{n_i}} x_{j,i} \right)  \times $\\
$~~~~\left( \alpha_0^{t}l_0 m_0 + \displaystyle{\sum_{j=1}^{n_t}} x_{j,t} + \displaystyle{\sum_{i=1}^{t-1}} \alpha_0^{t-i} \displaystyle{\sum_{j=1}^{n_i}} x_{j,i} \right)^T$. \\[6pt]
Thus, $\left( \mu_t, \Sigma_t \right) | \left( X_{1:t}, \alpha_0 \right) \sim NIW(m_t, l_t, \nu_t, \Psi_t)$.
\end{proof}

In cases where prior knowledge is lacking, it is important to consider non-informative priors. Among these, we recommend the use of the Jeffreys' prior \cite{jeffreys1998theory}. Its density function is proportional to the square root of the determinant of the Fisher information matrix, i.e., $\pi(\Theta) \propto |\mathcal{I}(\Theta)|^{1/2}$, while, a property, which is of key importance, is its invariance under a change of coordinates for the unknown parameters. In our set-up, we have $\pi(\lambda_t) \propto \lambda_t^{-1/2}$ and $\pi \left( \mu_t, \Sigma_t \right) \propto |\Sigma_t|^{-(d+2)/2}$ for the rate and the features' parameters, respectively. The total weight of information the posterior conveys to the observed sets arises from a geometric series, and specifically, when $t \rightarrow + \infty$, then for cardinality the total weight is of $ 1 / \left( 1 - \alpha_0 \right) $, and for the features the expected weight is of $ \lambda_t / \left( 1 - \alpha_0 \right) $. 

An illustrative example of the proposed PD posterior for monitoring the cardinality distribution $\rho(n)$ is depicted in Fig. \ref{fig:card}. Specifically, the figures shows the cardinality of the received RFS observations over time, with the corresponding posterior mean shown in blue. The initial 50 observations are simulated from a $\text{Pois}(10)$ distribution, followed by a smooth increase in the rate parameter $\lambda_t$ to 12 for the subsequent 30 observations. Finally, the rate $\lambda_t$ smoothly decreases to 5 for the last 20 observations. To assess the effect of the discounting factor we set three values for $\alpha_0$, and specifically $\alpha_0 \in \{ 0.8,0.9,1\}$, while a non-informative prior is employed. We observe that the smaller the $\alpha_0$, the more the posterior mean adapts to changes.

\begin{figure}
	\centering
	\includegraphics[width=\columnwidth]{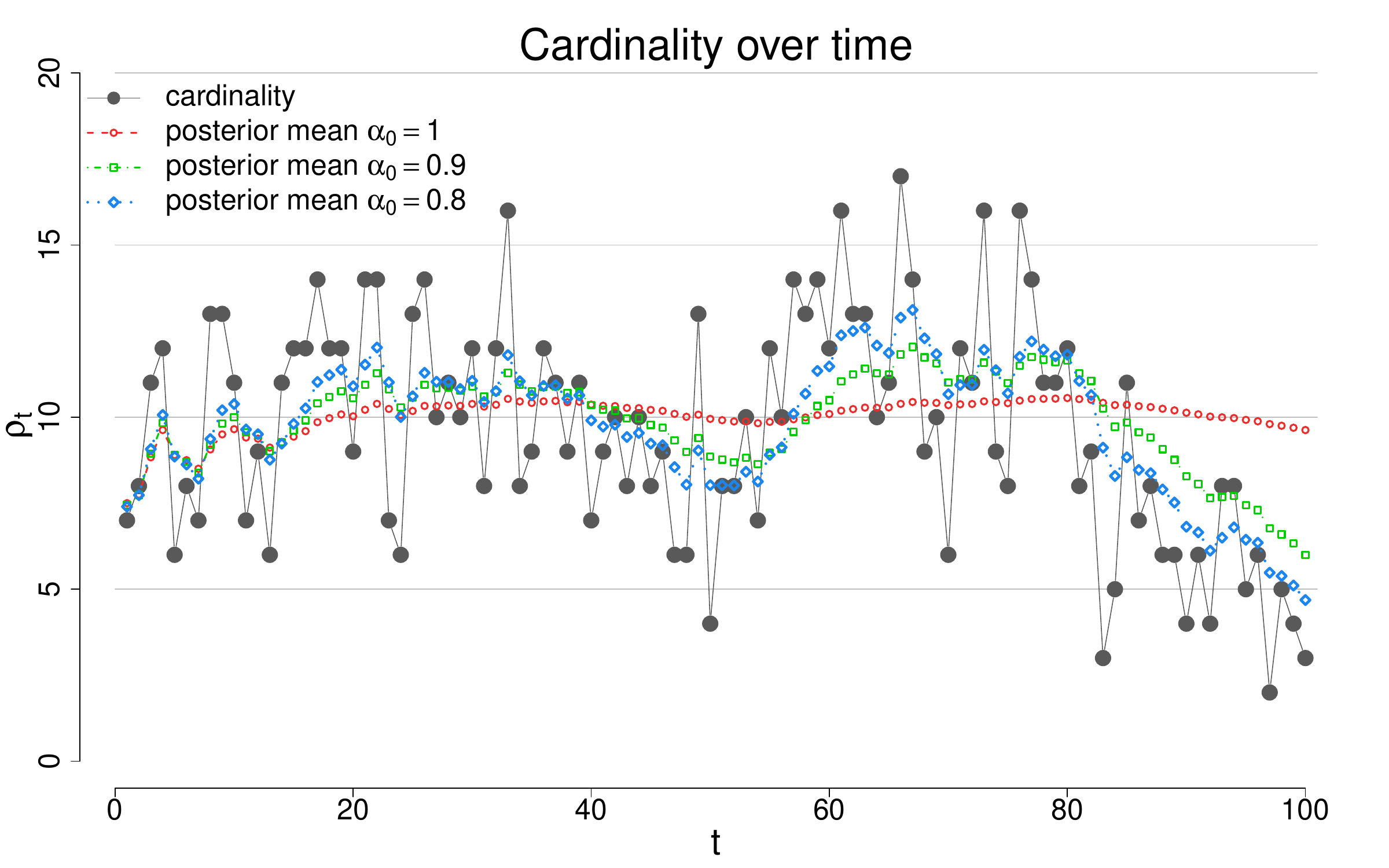}
	\caption{The figure illustrates the adaptive nature of the proposed approach i.e., how the proposed  approach tracks the cardinality, and its posterior mean in time for different values of $\alpha_0 \in \{ 0.8,0.9,1\}$. }
	\label{fig:card}
 \vspace{-5mm}
\end{figure}

\subsection{Detecting Anomalies in Sequential RFS Observations} \label{ssec:nov}

The proposed methodology for anomaly detection is based on the posterior predictive distribution, which will be the representative of the likelihood, taking into account the uncertainty for $\Theta_t=\left( \lambda_t, \mu_t, \Sigma_t\right)$. In our set-up, the general form of predictive density of $X_{t+1}$ is derived by integrating out the unknown parameters with respect to the posterior $p ( \Theta_t | X_{1:t}, \alpha_0)$. Specifically: 
{\small
\begin{align} \label{eq:predictive}
f(X_{t+1} | X_{1:t}, \alpha_0)  &=  \mathop{\mathlarger{ \mathlarger{\int_{\Theta_t }}} } f(X_{t+1} | \Theta_t) \pi ( \Theta_t | X_{1:t}, \alpha_0) d \Theta_t
\end{align}}
It is well known the non-uniformity issues of $f(X_{t+1} | \Theta_t)$ (and transition to $f(X_{t+1} | X_{1:t}, \alpha_0)$), which fails to peak at the most probable values. The Ranking Function (RF) was introduced in \cite{vo2018model} to alleviate these issues arising from the likelihood, but we tackle this problem in a different way. Specifically, considering the independence between the cardinality and the features given the cardinality, we will employ two independent predictive checks. These checks will assess anomalies for the cardinality and the features, respectively, based on their corresponding predictive distributions. Then, we will combine  the anomaly evidence of the two predictive checks into the Fisher score in Eq. \eqref{eq:fisher} for assessing the presence of anomaly in $X_{t+1}$ via \eqref{eq:alarm}. Starting from $n_{t+1}$, its posterior predictive distribution \cite{carlin2008bayesian} is a Negative Binomial: 
\begin{align} \label{eq:predcard}
 n_{t+1} | (n_{1:t}, \alpha_0) \sim NB \left( c_t, \dfrac{d_t}{d_t+1}\right). 
 \end{align}
To employ the predictive check for the cardinality, we define the set of Highest Predictive Probabilities (HPrP) as 
\begin{align} \label{eq:predR}
 \mathcal{R}(n_{t+1})  = \{ n: \rho(n | n_{1:t}, \alpha_0 ) > \rho( n_{t+1} | n_{1:t}, \alpha_0 )\}
 \end{align}
which is the set with the values of the posterior predictive with higher probability than the observed cardinality $n_{t+1}$, with $\mathcal{R}(n_{t+1})= \emptyset$ if $n_{t+1}$ is the mode of the posterior predictive density. Note that HPrP is closely related to the Highest Predictive Mass (HPrM), introduced in \cite{bourazas2022predictive} for a Bayesian approach in online anomaly detection for univariate processes. However, in our set-up, we need to define the probability that quantifies the discrepancy between the observed cardinality and its posterior predictive distribution and combine this probability with the corresponding probability of the features rather than employ a decision rule for the cardinality. Thus, we define the probability of obtaining results at least as extreme as $n_{t+1}$ by
\begin{align} \label{eq:prePDrc}
pr^{n}_{t+1} = \displaystyle{\sum_{n \notin \mathcal{R}(n_{t+1}) } \rho(n | n_{1:t}, \alpha_0)}. 
 \end{align}
The choice of HPrP in \eqref{eq:prePDrc} is optimal in the sense that minimizes the discrete measure $m(\mathcal{R}^c) = \sum_{i} \delta_{n_{i}} \left( \rho(n_{i} | n_{1:t}, \alpha_0) > \rho( n_{t+1} | n_{1:t}, \alpha_0 ) \right)$, where $\delta_{n_{i}}$ is the Delta Dirac function. In other words, HPrP is the shortest region that achieves the sum $\sum_{n \in \mathcal{R}(n_{t+1}) } \rho(n | n_{1:t}, \alpha_0)$, and consequently maximizes the region that indicates a discrepancy, an extremely useful property for anomaly detection. Essentially, $pr^{n}_{t+1}$ aggregates all the probabilities of values not belonging to $\mathcal{R}(n_{t+1})$, thus having smaller or equal probabilities than $n_{t+1}$. Large values for $pr^{n}_{t+1}$ imply that set $\mathcal{R}(n_{t+1})$ has few values, and thus $n_{t+1}$ has a relatively high probability in the distribution; small values are associated with extreme values in the posterior predictive distribution, essentially indicating an anomaly.

\begin{table*}[ht]
    \centering
    \begin{tabular}{lccccccr}
        \cline{2-7}
      {\Large (a)} &  $n_1=9$ & $n_2=7$ & $n_3=11$ & $n_4=10$ & $n_5=8$ & $n_6=16$ & \phantom{000}\\
        & $\bar{x}_1 = \begin{pmatrix} 0.21 \\ 0.05 \end{pmatrix}$ & $ \bar{x}_2 = \begin{pmatrix} 0.48 \\ -0.18 \end{pmatrix}$ & $ \bar{x}_3 = \begin{pmatrix} -0.30 \\ -0.15 \end{pmatrix}$ & $ \bar{x}_4 = \begin{pmatrix} 0.54 \\ 0.21 \end{pmatrix}$ & $ \bar{x}_5 = \begin{pmatrix} -0.26 \\ -0.09 \end{pmatrix}$ & $ \bar{x}_6 = \begin{pmatrix} 0.97  \\ 0.50 \end{pmatrix}$ \\
        \cline{2-7}
    \end{tabular}
\end{table*}

\begin{figure*}
	\centering
	\includegraphics[width=\textwidth]{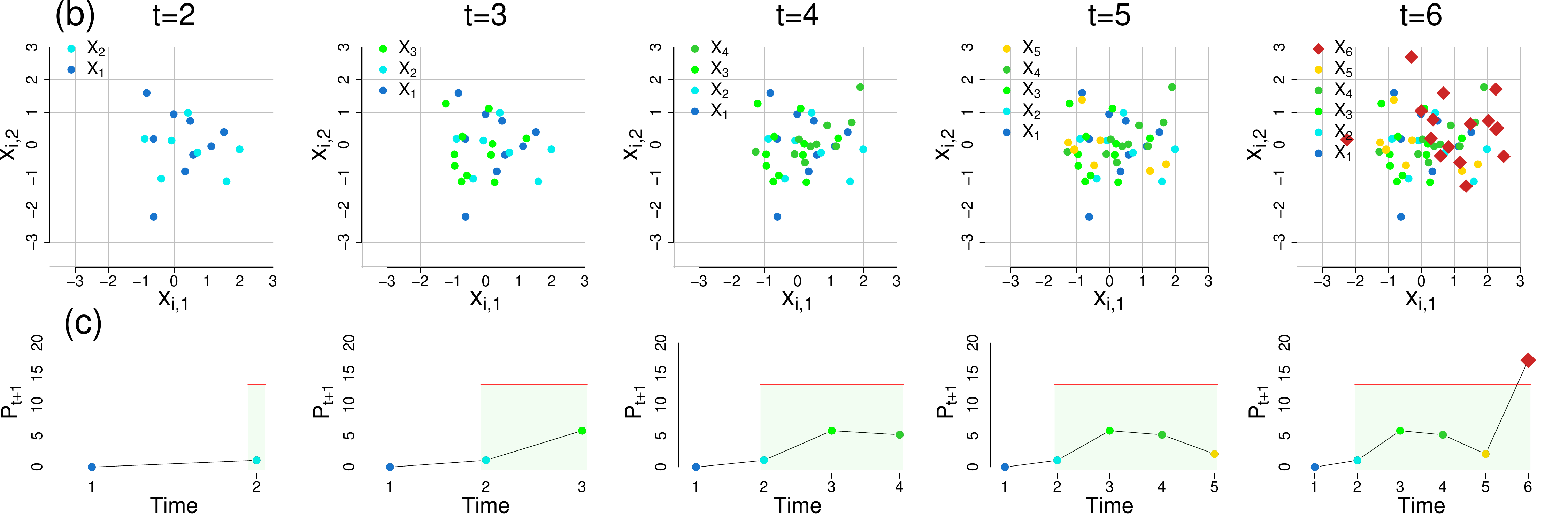}
	\caption{The figure depicts an illustrative example of the proposed approach. The top panel (a) displays the cardinality and sample mean vector of the observed RFSs over time, the center panel (b) shows scatterplots of the RFSs, and at the lower panel (c) we provide the Fisher score $P_{t+1}$ in Eq. \eqref{eq:fisher} based on the predictive checks. The Jeffreys' prior is used, while $\alpha_0=1$ for the PD in Eq. \eqref{eq:post}, and $\alpha=1/100$ for the quantile function in Eq. \eqref{eq:alarm}. In (c), the red line is the decision limit $q(1-\alpha)$, while the light green area indicates the no-anomaly region for the process. An anomalous RFS observation (marked with red $\diamond$) is detected at time-step 6 as shown in the figure.}
	\label{fig:ill}
 \vspace{-4mm}
\end{figure*}

Regarding the distribution of $pr^{n}_{t+1}$, it is discrete in the range $\left[ 0,1\right]$, as $\rho(\cdot | n_{1:t}, \alpha_0 )$ is discrete. However, under certain conditions, the distribution of $pr^{n}_{t+1}$ approximates asymptotically the standard uniform distribution $U(0,1)$. Specifically, it is well known that as $c_t \rightarrow +\infty$ and the probability of ``success'' $d_t / (d_t +1 ) \rightarrow 1$, i.e., when the posterior becomes informative, then the Negative Binomial approximates the Poisson with $\lambda_t = c_t / d_t$. For large values of $\lambda_t$, the $\mathcal{R}(n_{t+1})$ approximates the rejection region of a two-tailed test of Normal distribution, and  $pr^{n}_{t+1}$ approximates the classical two sided p-value, which follows a $U(0,1)$. But even in cases where the conditions of the approximation are not met, the distribution is uniform as possible due to its discreteness. Furthermore, the approximation would be mainly poor for values close to one, rather than close to zero, which are of main interest in detecting an anomaly, due to the large number of values with small probabilities in $\rho(\cdot | n_{1:t}, \alpha_0 )$. 

Continuing with the distribution of the features given the cardinality, we express an anomaly in terms of a disorder in the mean vector. From basic properties of the Normal distribution, it is known that the distribution of $\bar{x}_{t+1} $ is:
\begin{align} \label{eq:xbar}
\bar{x}_{t+1} \sim N \left( \mu_t, \Sigma_t / n_{t+1} \right). 
 \end{align}
Using the posterior $\pi\left( \mu_t, \Sigma_t | X_{1:t}, \alpha_0 \right)$, we derive the posterior predictive as in \eqref{eq:predictive} for the mean vector which is derived as in\cite{carlin2008bayesian}
\begin{align} 
\bar{x}_{t+1} | \left( X_{1:t}, \alpha_0 \right) \sim t_{\nu_t-d+1} \left( m_t, \dfrac{\lambda_t+1}{\lambda_t (\nu_t-d+1) n_{t+1}} \Psi_t \right), \notag
\end{align}\\
i.e., a $d$-variate $t$-Student distribution, with degrees of freedom $\nu_t-d+1$, mean vector $m_t$ and covariance matrix $\frac{\lambda_t+1}{\lambda_t (\nu_t-d+1) n_{t+1}} \Psi_t $. The predictive check for the features is based on the Hotelling statistic \cite{hotelling1947multivariate}, which is uses the Mahalanobis distance, and is admissible for testing anomalies in Normal data \cite{stein1956admissibility}. In our case, the Hotelling type statistic used will be the Mahalanobis distance between the observed sample mean vector and the posterior predictive mean vector, weighted by the posterior predictive covariance matrix, i.e.,
{\small
\begin{align} 
T^2_{t+1}= (\bar{x}_{t+1}-m_t)^T \left( \dfrac{(\lambda_t+1)d}{\lambda_t (\nu_t-d+1) n_{t+1}} \Psi_t\right)^{-1} (\bar{x}_{t+1}-m_t), \notag
\end{align}}
The distribution of $T^2_{t+1}$ test statistic is an $F$ with $d$ and $\nu_t-d+1$ degrees of freedom, i.e.,  $T^2_{t+1}\sim F_{d, \nu_t-d+1}$, with $\nu_t>d-1$ \cite{prins1997multivariate}. Significantly large values of $T^2_{t+1}$ indicate a large Mahalanobis distance, and consequently an potential anomaly. Thus, we define the probability of obtaining results at least as extreme as $\bar{x}_{t+1}$ by
\begin{align} \label{eq:prePDrf}
pr^{x|n}_{t+1} = 1 - CF_{d, \nu_t-d+1} \left( T^2_{t+1} \right), 
 \end{align}
where $CF_{d, \nu_t-d+1} \left( \cdot \right)$ is the cumulative distribution function of the F distribution with degrees of freedom $d$ and $\nu_t-d+1$. It is straightforward to prove that $pr^{x|n}_{t+1}\sim U(0,1)$ under the assumption of no anomaly. To combine the predictive checks for cardinality and features, we substitute the probabilities from \eqref{eq:prePDrc} and \eqref{eq:prePDrf} into \eqref{eq:fisher}, obtaining the Fisher test statistic $P_{t+1}$, which measures the extremeness of $X_{t+1}$ under its posterior predictive distribution. An alarm is triggered if $P_{t+1} > q(1-\alpha)$, where $q(\cdot)$ is the quantile function of the $\mathcal{X}^2_{4}$ distribution, and $\alpha$ is the false alarm rate. If $X_{t+1}$ is not OOC, it updates the posterior prediction for the next observation.

Figure \ref{fig:ill} visualizes a paradigm with a simultaneous change in the cardinality and the mean vector of the features. Precisely, for the first five RFS we generate $n_i \sim \text{Pois}(10)$, $i \in \{1,\ldots,5\}$, and $x_{j,i} \sim N_2(0_2,I_{2\times 2})$, i.e., a bivariate standard Normal distribution, where $j \in \{1,\ldots,n_i\}$. At the sixth RFS, we introduce an anomaly by increasing the rate parameter of the cardinality to 16 and adding one standard deviation to the first component of the mean vector and half a standard deviation to the second component, i.e., $x_{j,6} \sim N_2((1\;0.5)^T,I_{2\times 2})$. The Jeffreys' prior is used, while we set the discounting factor $\alpha_0=1$ (i.e., no discounting) and the false alarm rate $\alpha=1/100$. As we observed, the anomaly is detected at time $t=6$, as $X_{6}$ deviates significantly from the previous for both the cardinality ($n_{6}=16$) and the mean vector ($\bar{x}_6=(0.97\;0.50)^T$).

\section{Evaluation} \label{sec:sim}

In this section, we compare the performance of the proposed Bayesian predictive checks (PC) methodology against the ranking function (RF) introduced in \cite{vo2018model} under various scenarios of anomalies in the sequence of Poisson RFS observations. Regarding the IC process, we generate 10,000 batches of $T=30$ RFSs where the cardinality follows a $\text{Pois}(10)$, and the features follow the $N_2\left(0_2, I_{2\times2}\right)$, i.e. the bivariate standard Normal distribution. For the case of OOC data, we introduce anomalies in the IC sequences in each time-step $t \in \{2,\ldots,T\}$, drawing samples from the following 5 scenarios:  

\begin{enumerate}
    \item Spatial Density: we draw from a Normal with mean vector $\mu_t'=(1\; 1)^T$, i.e., we introduce a shift size of one standard deviation for each component in the mean vector (Fig. 3(a)).
    \item Cardinality Distribution: we draw samples from a Poisson with rate $\lambda'=20$, i.e., we introduce an increase to the cardinality rate (Fig. 3(b)).
    \item Cardinality Distribution: we draw samples from a Poisson with rate $\lambda'=2$, i.e., we introduce a decrease to the cardinality rate (Fig. 3(c)).
    \item Spatial Density and Cardinality: we draw samples from a Poisson with rate $\lambda'=15$ and a Normal with mean vector $\mu_t'=(1\; 1)^T$, i.e., we introduce a moderate increase to the cardinality rate and a moderate shift for the mean vector (Fig. 3(d)). 
    
    \item Spatial Density and Cardinality: we draw from a Poisson with rate $\lambda'=5$ and a Normal with mean vector $\mu_t'=(1\; 1)^T$, i.e., we introduce a moderate decrease to the cardinality rate and a moderate shift for the mean vector (Fig. 3(e)). 
    
\end{enumerate}

It is important to note that the 5 investigated OOC scenarios pose a significant challenge due to subtle deviations in cardinality distribution and spatial density of the features,  relative to the IC process. Next, we demonstrate how our proposed approach handles these challenging scenarios compared to existing state-of-the-art methods.



\begin{figure}
	\centering
	\includegraphics[width=0.99\columnwidth]{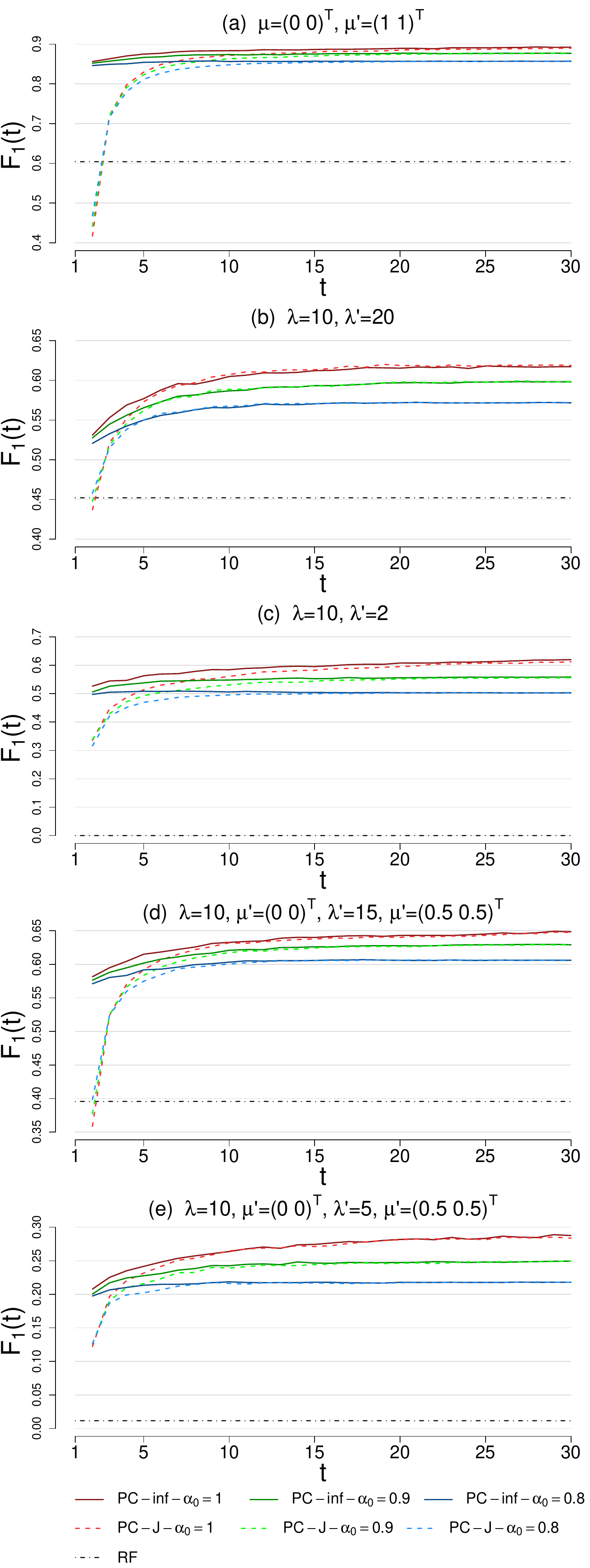}
	\caption{The $F_1(t)$ scores for $t\in \{2,\ldots,30\}$ of the proposed predictive checks (PC) and ranking function (RF) for 5 scenarios. }
	\label{fig:sim}
\end{figure}


PC requires the definition of a prior distribution, so, within this simulation study, we will take the opportunity to examine its sensitivity under the absence or the presence of prior information. Precisely, we will use the non-informative Jeffreys' prior (denoted as $J$) while for the informative prior setting (denoted as $inf$) we assume $\lambda \sim G(50.5,5)$ and $(\mu_t, \Sigma_t) \sim N(m_0 = 0_2, l_0=50, \nu_0=48, \Psi_0 = 49 \cdot I_{2\times 2})$, where $0_2$ is the zero vector and $I_{2\times 2}$ is the two-dimensional identity matrix. Note that this is the resulting posterior on average for a process with five IC RFSs using the Jeffreys' prior and setting $\alpha_0=1$. Furthermore, to assess the effect of the discounting parameter $\alpha_0$ on the performance, we set $\alpha_0 \in \{ 0.8, 0.9, 1\}$. Thus, we have $2 \times 3 = 6$ versions of PC for the different choices of the prior and $\alpha_0$. Regarding the competing method RF (shown as a dotted black line in Fig. 3), we use the theoretical values of the IC process to achieve the best possible performance. In other words, for the RF method we assume that the parameters are completely known (not estimated from a training dataset), while for PC are estimated by the information from  the prior (if available) and the IC observations until the time of an anomaly, e.g., for a test at time three we have the information from  only two RFS observations.

For a fair comparison between the two methods, we set the false alarm rate $\alpha=1/100$ for the PC, and we equivalently set the quantile of the distribution of the ranking function under the IC process to the $0.01^{th}$ level, in order to derive the corresponding decision limits. As a performance metric we calculate the $F_1(t)$, given by:
\begin{equation}
F_1(t)=2 \cdot tp_t/(2\cdot tp_t+fp_t+f_n),    
\end{equation}
where $tp_t$ and $f_n$ are the percentages of true positives and false negative for an OOC sequence, respectively, and $fp_t$ is the percentage of false positives (or false alarms) in an IC sequence at each time $t \in \{2,\ldots,30\}$. Note, that we do not provide any test for $t=1$, as PC starts testing from $t=2$, ``sacrificing'' the first observation to obtain the posterior distribution. 
As shown in Figure \ref{fig:sim}, the PC learning process enhances detection performance, particularly in non-informative settings, with prior information being especially useful when the observation horizon is short. Regarding the values of the discounting factor, $\alpha_0=1$ consistently performs better by using all observed RFS information, but other values offer comparable performance with more adaptation flexibility. Comparing methods, PC outperforms RF, achieving higher $F_1$ scores across all scenarios. Even when RF performs well (e.g., scenarios 1 and 2), PC, with a non-informative prior, requires only two or three IC observations to surpass it, assuming known parameters. When the cardinality rate decreases, RF appears incapable of detecting anomalies.
The ranking function, a weighted product of the cardinality distribution and feature likelihood, determines alarm thresholds. We raise an alarm when the ranking function's score falls below a threshold derived from the IC ranking function's quantile. However, a structural limitation of the ranking function is that when cardinality decreases, the number of likelihood factors also decreases, preventing their product from reaching very low values, even in the presence of anomalies. This limits its sensitivity, making it less effective in such cases. On the other hand, PC remains robust due to its axiomatic anomaly detection framework. Exploring robustness to distributional violations would be an interesting extension, as Bayesian methods have been shown to be resilient to such violations, particularly with informative priors \cite{bourazas2022predictive}.


\section{Conclusion} \label{sec:disc}

In this work, we developed a methodological framework to efficiently detect anomalies online for the Poisson point-patterns process while relaxing the assumption of known parameters. We deviated from the conventional approach of separating the training and testing phases; instead, we proposed a Bayesian self-starting process that sequentially estimates the unknown parameters while assessing the conformity of new sets. Additionally, we introduced a new class of power discounting posterior distributions alongside posterior predictive checks, enabling adaptive learning and robustness in detecting anomalous observations.






\balance
\bibliographystyle{ieeetr}
\bibliography{main}

\end{document}